\documentclass[twoside,11pt]{article}

\usepackage[preprint]{jmlr2e}
\usepackage{amsmath}
\usepackage{bm}
\usepackage{booktabs}
\usepackage{microtype}
\usepackage[section]{placeins}
\graphicspath{{figures/publication/}{results/publication_robustness/}{./}}
\hypersetup{hidelinks}

\newtheorem{independentlemma}{Lemma}
\newtheorem{independentcorollary}{Corollary}

\newcommand{\R}{\mathbb{R}}
\newcommand{\E}{\mathbb{E}}
\newcommand{\Prob}{\mathbb{P}}
\newcommand{\1}{\mathbf{1}}
\newcommand{\diag}{\operatorname{diag}}
\newcommand{\tr}{\operatorname{tr}}
\newcommand{\op}{\mathrm{op}}
\newcommand{\F}{\mathrm{F}}
\newcommand{\clip}{\operatorname{clip}}
\newcommand{\Agg}{\operatorname{Agg}}

\ShortHeadings{Invariant Representation Measurements}{Karasan}
\firstpageno{1}

\begin{document}\title{Representation Measurements Under Function-Preserving Reparameterizations}

\author{\name Abdullah Karasan \email akarasan@umbc.edu\\
  \addr University of Maryland, Baltimore County\\
  1000 Hilltop Circle, Baltimore, MD 21250, USA}

\maketitle

\begin{abstract}%
Hidden coordinates are not uniquely determined by a language model's
input--output function, so representation-derived measurements should be
invariant to function-preserving changes of basis. This study shows that
column-permutation parallel analysis violates function-preserving reparameterization invariance because its
reference distribution and selected component count can change while the model
function and observed covariance spectrum remain fixed. More generally, a
data-internal reference procedure cannot simultaneously preserve every
coordinate marginal, remain orthogonally equivariant, and remove
cross-coordinate covariance. Empirically, across five models, three retrieval domains, and
75 transformations, median component-count disagreement is 0.79 and median
fixed-threshold decision disagreement is 0.26. A centering-only control
isolates the reference-driven effect, with 1,141 of 1,200 component counts changing despite an
unchanged observed spectrum, whereas independent parallel analysis seeds change none of the
corresponding decisions. By contrast, orthogonally invariant comparator scores remain numerically stable with
similar held-out discrimination. Together, these results show that parallel analysis-derived component
counts and decisions can reflect hidden-coordinate choice rather than a
well-defined property of the model.
\end{abstract}


\section{Introduction}
\label{sec:introduction}

Internal representations are routinely summarized by estimated ranks,
subspaces, and projection scores. Such quantities are then interpreted as
properties of a learned model or as evidence about downstream behavior.
However, hidden coordinates are generally not unique. A model can often be
reparameterized by an invertible change of basis while preserving its
input--output function. Because these parameterizations describe the same
model function, a representation statistic intended to characterize that
function should return the same value for all of them.

These summaries are not inert descriptions. A selected component count fixes
the subspace on which a probe is fitted, the dimension retained in a low-rank
approximation, and, in the pipeline studied here, whether a retrieved context
is classified as containing the evidence needed to answer a question. If the
count reflects the coordinate system rather than the model, so do the decisions
that follow from it.

This study examines that possibility in a retrieved-context measurement
pipeline. Given token representations of a context and a fixed target answer,
the pipeline uses column-permutation parallel analysis (PA) to select a context
subspace and scores the answer by its projection onto that subspace. The
resulting score distinguishes evidence-containing contexts from matched
evidence-ablated contexts. Column permutation is attractive because it
preserves each coordinate's empirical marginal distribution while disrupting
cross-coordinate dependence. That is also its vulnerability: marginals are
properties of the chosen basis, not of the model function.

The mismatch is sharp. An orthogonal reparameterization leaves the observed
covariance spectrum exactly unchanged because \(XR\) and \(X\) have identical
singular values. It nonetheless changes the diagonal of the covariance matrix,
and the permutation reference is built from precisely that diagonal. The
observed quantity and the reference against which it is compared therefore
transform differently under a change of basis that the model itself does not
distinguish.

The central question is whether the PA-based measurement remains unchanged
across functionally equivalent parameterizations of the same model. It is
addressed at three levels. First, coordinate dependence is isolated in raw
column-permutation PA without positing a unique ``true'' rank of language-model
activations. Next, the analysis tests whether this dependence persists through
the standardization and minimum-component rules of the implemented scorer.
Finally, it determines whether changes in the resulting score alter
fixed-threshold classifications of the two context conditions. The pipeline
serves as a representative case study rather than a proposed retrieval
evaluator. The requirement it is used to test, that a statistic interpreted as
a property of a model function be well-defined across function-preserving
hidden-coordinate systems, applies to representation-derived statistics more
generally.

The analysis yields four contributions.
\begin{itemize}
\item A measurement-validity criterion for representation-derived statistics.
I formalize the requirement at the exact-equality limit: a statistic interpreted
as a property of a model function must be invariant under
architecture-compatible, function-preserving reparameterization. Because such
reparameterizations hold both the input and the model output fixed, they serve
as controlled interventions in which any change in the statistic isolates
dependence on the hidden coordinate system.

\item A structural limitation of data-internal reference procedures. I prove
that no randomized reference procedure constructed solely from the observed
representation matrix can simultaneously preserve each coordinate's empirical
marginal distribution, remain equivariant under right-orthogonal
transformations, and remove cross-coordinate covariance. Column-permutation PA
is one instance of this broader incompatibility.

\item A spectrum-conditioned range bound for column-permutation PA. I derive
the expected permutation-reference covariance and show that the range of
PA-selected component counts over the orthogonal orbit admits a lower bound
computable from the observed spectrum, even though that spectrum is itself
right-orthogonally invariant. Separate constructions cover the full orthogonal
class used for the root mean square normalization (RMSNorm) models and the
mean-preserving subgroup used for the LayerNorm model. Evaluated on 3,600
context--layer spectra, the bound certifies a nonzero orbit-wide range for
15.2\% of contexts and is inconclusive, rather than contradicted, elsewhere;
the bound is a sufficient certificate rather than a sharp characterization.

\item A controlled empirical study separating reference-driven instability
from its alternatives. I show that, across five models, four architecture
families, three retrieval domains, and five reparameterizations per model,
function-preserving reparameterization changes fixed-threshold classification
of the evidence-containing and evidence-ablated conditions at a median rate of
0.26, whereas an independent-seed rerun without reparameterization changes none
of 1,200 matched decisions. A centering-only arm isolates the mechanism
directly: the observed covariance spectrum is exactly invariant, yet 1,141 of
1,200 contexts change selected component count. The effect persists across 72
combinations of permutation budget, reference quantile, selection rule, and
layer aggregation. Right-orthogonally invariant comparator scores remain
numerically unchanged and attain comparable observed median AUROC across the 15
model--domain combinations.
\end{itemize}

\section{Related Work}
\label{sec:related-work}

The relevant literature connects representation measurement to model
symmetries, dimensionality estimation, and retrieval evaluation.

\subsection{Representation Invariance and Measurement}

Neural-network symmetries provide the first foundation for the invariance
criterion. Model parameters are not unique descriptions of the input--output
function. Hidden units can often be permuted, and positively homogeneous
nonlinearities admit compensating rescalings of adjacent layers. These
symmetries affect loss landscapes, Bayesian posteriors, and comparisons between
independently trained networks
\citep{dinh2017sharp,ainsworth2023gitrebasin,godfrey2022symmetries}. The
sharpness construction of \citet{dinh2017sharp} is a particularly close
precedent because a function-preserving rescaling changes a diagnostic
quantity without changing generalization. The present setting extends that
logic from a diagnostic of the loss landscape to a representation measurement
whose variation propagates to downstream decisions. Recent compression methods
exploit a stronger form of the same freedom in Transformers. QuaRot inserts
fixed orthogonal transformations into the residual stream and absorbs them
into adjacent weights \citep{ashkboos2024quarot}, while SpinQuant learns such
transformations without changing the full-precision model function
\citep{liu2025spinquant}. Here, the same coordinate freedom is repurposed as
a controlled intervention: because the input and model output are fixed, any
change in the statistic isolates hidden-coordinate dependence.

A second foundation comes from representation comparison. Representational
similarity analysis (RSA) compares pairwise stimulus relations rather than
individual activation coordinates \citep{kriegeskorte2008rsa}. Centered kernel
alignment (CKA) similarly compares example--example Gram matrices and is
invariant to orthogonal transformations and isotropic scaling
\citep{kornblith2019cka}. Subsequent work has emphasized that a measure's
invariance class is part of its scientific meaning rather than a technical
afterthought. \citet{ding2021grounding} evaluate similarity measures through
sensitivity to functionally relevant changes and specificity to irrelevant
ones, while \citet{hayne2024function} show that common similarity measures can
track linearly decodable information more closely than information actually
used by the network. The ReSi benchmark extends such comparisons across
language, vision, and graph models using explicit groundings of similarity
\citep{klabunde2025resi}. This literature motivates the right-orthogonally
invariant comparators used in the present study.

Identifiability and interpretability studies sharpen the same measurement
concern. Learned representations may be determined only up to invertible
linear transformations \citep{roeder2021identifiability}, and models with
nearly identical output distributions need not have nearby representations
under the relevant equivalence class \citep{nielsen2025closeness}. Probe
accuracy may reflect probe capacity rather than information uniquely
attributable to a representation, motivating control tasks and explicit
estimands \citep{hewitt2019probes,pimentel2020probing}. Likewise,
individual-neuron conclusions can depend on the analysis procedure
\citep{antverg2022neurons}, while sparse-autoencoder features need not form a
unique or complete set of canonical units \citep{leask2025sae}. The present
study complements these concerns by showing that even an unsupervised spectral
statistic can vary between exact reparameterizations of one model and that the
variation can propagate to classification decisions.

Finally, statistical work on permutation methods makes clear that a
permutation reference is valid only relative to its invariance or
exchangeability assumptions. Standard permutation tests can lose validity
when the tested null is weaker than equality of the underlying distributions,
although appropriate studentization can restore asymptotic validity
\citep{chung2013permutation}. Exactness with a finite set of randomly sampled
permutations likewise depends on the transformation-group construction and the
definition of the randomized test \citep{hemerik2018exact}. These studies
concern hypothesis-test validity rather than representation reparameterization,
but they establish the broader principle that a permutation distribution is
not automatically neutral. The result developed here addresses a different
requirement: whether a data-internal permutation reference transforms
consistently when hidden coordinates are changed without altering the model
function.

\subsection{Dimensionality Estimation in Language Models}

The possibility that a representation statistic depends on the hidden coordinate system is especially important for dimensionality estimation, because dimensionality has become a common descriptor of language-model representations.
Recent work increasingly uses intrinsic dimension to characterize the geometry
and development of language-model representations. Across several transformer
families, \citet{cheng2025abstraction} identify a high-dimensional middle-layer
phase associated with linguistic abstraction, transferability, and
cross-model predictability. Studying training trajectories,
\citet{lee2025compositionality} find that nonlinear intrinsic dimension tracks
semantic compositionality, whereas linear dimensionality captures different,
more superficial properties of the input. These studies suggest that dimensionality captures meaningful properties of language-model representations, but they also make the choice of estimator important. Intrinsic manifold dimension, covariance rank, effective rank, and factor count quantify different properties and need not agree. Whether each estimand is invariant to a change of hidden basis is a separate question, and the measurement-validity criterion developed here applies to each of them. This study targets
the component count selected by permutation PA rather than nonlinear manifold
dimension. The analysis therefore examines a prerequisite for interpreting
that component count as a model-level quantity, namely invariance across exact
reparameterizations of the same input--output function.

\subsection{Context Sufficiency and Retrieval Evaluation}

This coordinate-dependence problem is examined in retrieval, where distinguishing
useful evidence from inadequate context is practically important.
Contemporary retrieval evaluation separates the quality of retrieved evidence
from the generator's ability to use it. ARES estimates context relevance,
answer faithfulness, and answer relevance with calibrated model-based judges
\citep{saadfalcon2024ares}, while \citet{ru2024ragchecker} provide claim-level
diagnostics for retrieval and generation and validate them against human
judgments. \citet{joren2025sufficient} define sufficient context
as retrieved text that
contains enough information to answer a query and show that the evidence-sufficiency distinction helps diagnose errors and guide abstention.

Guided by this evidence-centered view, evidence-containing
contexts are treated as sufficient and matched evidence-ablated contexts as
insufficient. These operational labels do not independently establish
semantic sufficiency in the broader sense of \citet{joren2025sufficient}.
A general evaluator for end-to-end retrieval systems is not proposed. Instead,
the question is whether an internal-representation score that distinguishes the two
conditions defined by the data construction constitute a well-defined measurement of a fixed
model. Held-out AUROC measures discrimination, while invariance under
function-preserving reparameterization measures dependence on the hidden
coordinate system.

\section{Setup and Theoretical Analysis}
\label{sec:theory}

This section defines the representation matrices, admissible
reparameterizations, and invariance criterion. It then analyzes
column-permutation PA, establishes a limitation for marginal-preserving
reference distributions, and introduces invariant comparator scores.

At layer \(\ell\), let \(X_\ell\in\R^{n_\ell\times d}\) contain the
retrieved-context token representations and let
\(A_\ell\in\R^{m_\ell\times d}\) contain the representations of a fixed target
answer.  A change of hidden coordinates acts on both matrices as
\[
  (X_\ell,A_\ell)\longmapsto(X_\ell R,A_\ell R),
  \qquad R\in O(d).
\]
Here \(O(d)=\{R\in\mathbb{R}^{d\times d}:R^\top R=I_d\}\) is the
orthogonal group. Right-multiplying the hidden representations by \(R\)
changes their coordinate system while preserving inner products and Euclidean
geometry. If the adjacent model parameters are
transformed accordingly, this change of coordinates leaves the model's
input--output function unchanged and therefore defines a function-preserving
reparameterization. For the RMSNorm models considered here, coordinatewise gain parameters can be
absorbed into adjacent linear maps, allowing any \(R\in O(d)\). For the
layer-normalization (LayerNorm) model, the transformation must also preserve
the mean direction. Its admissible class is
\[
  \mathcal G_{\mathrm{LN}}
  =\{R\in O(d):R\mathbf 1=\mathbf 1\},
\]
which is a proper subgroup of the full orthogonal class admitted by the
RMSNorm models. Theorem~\ref{thm:single-view-limit} concerns the full group
and therefore applies directly to the RMSNorm setting, whereas
Theorem~\ref{thm:same-spectrum} below gives a separate construction within
\(\mathcal G_{\mathrm{LN}}\). This coordinate freedom is referred to as
function-preserving reparameterization.

With these architecture-specific classes fixed, a statistic
\(T(X,A)\) is called invariant under the specified function-preserving
reparameterization class when
\[
  T(XR,AR)=T(X,A)
  \qquad\text{for every }R\text{ in the specified architecture-compatible class}.
\]
Given a threshold \(\tau\), the induced decision
\(D_\tau(X,A)=\mathbf 1\{T(X,A)\geq\tau\}\) is invariant if it is unchanged
under the same transformation class.

Function-preserving reparameterization invariance is necessary when \(T\) is interpreted as a
property of the input--output function or the retrieved evidence, rather than as a
property of one chosen hidden-coordinate system.

\subsection{Coordinate Dependence of Column-Permutation Parallel Analysis}
\label{sec:permpa-theory}

Applying the function-preserving reparameterization invariance requirement to
column-permutation PA reveals a mismatch because the observed spectrum is right-orthogonally invariant, but the permutation
reference distribution need not be.

PA is considered first because its mechanism can be analyzed explicitly within both
normalization architectures. After establishing the PA-specific result,
Section~\ref{sec:single-view-limit} shows that, under
the full orthogonal group, the same issue reflects a broader incompatibility
between exact marginal preservation and equivariance.

To see the source of the mismatch, let \(X\in\R^{n\times d}\), \(n\geq2\),
have centered columns and define
\[
  S=\frac{1}{n}X^\top X.
\]
Here, \(\lambda_k(A)\) denotes the \(k\)-th largest eigenvalue of a symmetric matrix \(A\). For each feature \(j\), independently draw a uniform \(n\times n\) permutation
matrix \(P_j\) and set \(X^\pi_{\cdot j}=P_jX_{\cdot j}\).
The nonparametric column-permutation form of parallel analysis is considered
\citep{horn1965parallel,buja1992parallel,dobriban2020permutation}. It compares
the ordered eigenvalues of \(S\) with componentwise quantiles of
\[
  S^\pi=\frac{1}{n}(X^\pi)^\top X^\pi.
\]
For ideal permutation-reference quantiles \(t_k(X)\), the estimator retains consecutive
components from the largest eigenvalue onward and stops at the first failed
comparison:
\[
  \widehat r_{\mathrm{PA}}(X)
  =
\max\bigl\{r\in\{0,\ldots,\min(n,d)\}:
\lambda_k(S)>t_k(X)
\text{ for all }1\leq k\leq r\bigr\}.
\]
Here \(\widehat r_{\mathrm{PA}}(X)\) denotes the number of leading components selected by PA, and \(r\) is a candidate component count in the maximization.

Both lemmas below follow from elementary properties of sampling without
replacement and are stated because the later results depend on their exact
finite-sample form.

\begin{independentlemma}[Expected Permutation Covariance]
\label{lem:orientation}
Conditionally on \(X\),
\[
  \E_\pi[S^\pi]=\diag(S).
\]
For every \(R\in O(d)\), \(X\) and \(XR\) have the same singular values, but
\[
  \E_\pi[S^\pi_R]=\diag(R^\top S R),
\]
where \(S^\pi_R\) is generated by independently permuting the columns of
\(XR\). Thus the expected permutation-reference covariance depends on the
hidden-coordinate system even though the observed covariance eigenvalues do
not.
\end{independentlemma}

\begin{proof}
Write \(x_j=X_{\cdot j}\).  The diagonal entries satisfy
\[
  \frac{1}{n}(P_jx_j)^\top(P_jx_j)
  =
  \frac{1}{n}\lVert x_j\rVert_2^2=S_{jj}.
\]
For \(j\neq k\), independence of \(P_j\) and \(P_k\), together with column
centering, gives
\[
  \E_\pi[(P_jx_j)^\top(P_kx_k)]=0.
\]
This proves the first identity.  Applying it to \(XR\) proves the second,
whereas \((XR)^\top(XR)/n=R^\top S R\) is orthogonally similar to \(S\).
\end{proof}

Lemma~\ref{lem:orientation} identifies the source of coordinate dependence at
the level of the reference mean. To connect that mean to the componentwise PA
thresholds, the next lemma controls the fluctuation of a realized permutation
covariance around its coordinate-dependent center.

\begin{independentlemma}[Permutation Covariance Bound]
\label{lem:null-concentration}
Define
\[
  V(S)
  =
  (\tr S)^2-\sum_{j=1}^d S_{jj}^2.
\]
Then, conditionally on \(X\),
\[
  \E_\pi\!\left[
    \left\lVert S^\pi-\diag(S)\right\rVert_\F^2
  \right]
  =
  \frac{V(S)}{n-1}.
\]
Consequently, for every \(\delta\in(0,1)\), with conditional probability at
least \(1-\delta\),
\[
  \left\lVert S^\pi-\diag(S)\right\rVert_\op
  \leq
  \varepsilon_\delta(S)
  :=
  \sqrt{\frac{V(S)}{(n-1)\delta}}.
\]
If \(s_{(1)}\geq\cdots\geq s_{(d)}\) are the ordered diagonal entries of
\(S\), then simultaneously for all \(k\),
\[
  \left|\lambda_k(S^\pi)-s_{(k)}\right|
  \leq\varepsilon_\delta(S).
\]
For a real-valued random variable \(Y\), define the upper
\((1-\alpha)\)-quantile by
\[
  Q_{1-\alpha}(Y)=\inf\{y\in\R:\Prob(Y\leq y)\geq1-\alpha\}.
\]
Here \(y\) is a candidate real-valued threshold.
Then the ideal componentwise threshold
\(t_k(X)=Q_{1-\alpha}\{\lambda_k(S^\pi)\}\) satisfies
\[
  t_k(X)\leq s_{(k)}+\varepsilon_\alpha(S).
\]
\end{independentlemma}

The proof is provided in Appendix~\ref{app:null-concentration}.
The \(1/\sqrt{\delta}\) factor comes from Markov's inequality applied to the
squared Frobenius norm; at \(\delta=0.05\) it is
\(\sqrt{20}\approx4.47\) times the root-mean-square Frobenius fluctuation.
The additional substitution
\(\lVert\cdot\rVert_\op\leq\lVert\cdot\rVert_\F\) can only increase
conservatism. Accordingly, the bound is a sufficient finite-sample guarantee,
not an estimate of the actual 95th-percentile deviation.

Having linked the permutation thresholds to the covariance diagonal, the analysis uses
the Schur--Horn theorem to determine how much that diagonal can change under a
rotation.
If \(S\succeq0\),
\(\tr(S)>0\), and \(\bar\lambda=\tr(S)/d\), then
\[
  \min_{R\in O(d)}
  \lambda_1\!\left(\E_\pi[S_R^\pi]\right)
  =\bar\lambda,
  \qquad
  \max_{R\in O(d)}
  \lambda_1\!\left(\E_\pi[S_R^\pi]\right)
  =\lambda_1(S).
\]
The lower endpoint is attained by a basis with constant covariance diagonal;
the upper endpoint is attained by aligning one coordinate with a leading
eigenvector. Hence the largest eigenvalue of the expected reference
covariance, not the mean or a quantile of the largest realized
permutation-reference
eigenvalue, can vary over an exact
multiplicative range of
\[
  \frac{d\,\lambda_1(S)}{\tr(S)}.
\]
Importantly, this expectation range alone does not establish that PA thresholds
or selected component counts must vary. To obtain such a statement, the argument combines
the Schur--Horn construction with the finite-sample quantile bound from
Lemma~\ref{lem:null-concentration}. The resulting spectrum-dependent
quantities are
\[
  c_\alpha(S,n)
  =
  \bar\lambda+
  \sqrt{
    \frac{(\tr S)^2-(\tr S)^2/d}
         {(n-1)\alpha}
  }
\]
and define
\[
  m_\alpha(S,n)
  =
  \max\left(
  \{0\}\cup\left\{
    m\in\{1,\ldots,\min(n,d)\}:
    \lambda_m(S)>c_\alpha(S,n)
  \right\}\right).
\]

Here \(m_\alpha(S,n)\) counts the leading observed eigenvalues that remain
strictly above a conservative reference threshold in a constant-diagonal
orientation. This gives the following orbit-wide guarantee.

\begin{theorem}[Component-Count Range]
\label{thm:spectral-range}
Consider ideal PA based on independent column permutations. Use upper
\(1-\alpha\) componentwise quantiles, selection from the largest component
until the first failed comparison, and strict exceedance. Then
\[
  \max_{R\in O(d)}\widehat r_{\mathrm{PA}}(XR)
  -
  \min_{R\in O(d)}\widehat r_{\mathrm{PA}}(XR)
  \geq
  m_\alpha(S,n).
\]
The left-hand side is the range of PA-selected component counts over all
orthogonal coordinate systems \(R\in O(d)\). In particular, an orientation aligned with a leading eigenvector has selected
component count zero, while a constant-diagonal orientation retains at least
\(m_\alpha(S,n)\) components.
\end{theorem}

Appendix~\ref{app:spectral-range} gives the proof. Because
Lemma~\ref{lem:null-concentration} is conservative, \(m_\alpha(S,n)\) can be
zero even when the actual PA count varies across orientations. For the main
setting \(d=896\), \(n=192\), and \(\alpha=0.05\), the bound gives
\(c_\alpha/\bar\lambda\approx291\). Thus a positive certificate requires
approximately \(\lambda_1(S)\gtrsim0.32\,\tr(S)\), which explains why the
sufficient bound certifies only a subset of contexts. Whenever
\(m_\alpha(S,n)>0\), the theorem nevertheless guarantees an orbit-wide range
directly from the observed spectrum. The range result uses the full group \(O(d)\). An explicit construction that
also lies within the LayerNorm-compatible subgroup is given next.

\begin{theorem}[Different PA Component Counts]
\label{thm:same-spectrum}
Fix integers \(n\geq2\) and \(d\geq3\), \(\lambda>0\), and
\(\alpha\in(0,1)\). Let
\(u\in\R^n\) be any centered unit vector, \(e_1\) the first coordinate vector,
and
\[
  w=(2/3,\,2/3,\,-1/3,\,0,\ldots,0)^\top.
\]
Define
\[
  X_{\mathrm{loc}}=\sqrt{n\lambda}\,u e_1^\top,
  \qquad
  X_{\mathrm{mix}}=\sqrt{n\lambda}\,u w^\top.
\]
There exists \(R\in\mathcal G_{\mathrm{LN}}\) such that
\(X_{\mathrm{mix}}=X_{\mathrm{loc}}R\); hence the matrices are related by a
LayerNorm-compatible transformation and have the same singular values.
Nevertheless, ideal PA with consecutive-from-largest selection and strict
exceedance satisfies
\[
  \widehat r_{\mathrm{PA}}(X_{\mathrm{loc}})=0.
\]
If
\[
  n-1>\frac{48}{25\alpha},
\]
then
\[
  \widehat r_{\mathrm{PA}}(X_{\mathrm{mix}})=1.
\]
\end{theorem}

The proof is provided in Appendix~\ref{app:same-spectrum}. Together, the two
theorems provide complementary conclusions.
Theorem~\ref{thm:spectral-range} gives a spectrum-dependent range guarantee
over the full orthogonal orbit, whereas Theorem~\ref{thm:same-spectrum} gives
a concrete count change within the LayerNorm-compatible subgroup. In both
cases, the observed spectrum remains fixed while the permutation reference
changes with the basis.

The theorems above analyze ideal PA with consecutive-from-largest selection and
strict exceedance, a rule that may return zero. They establish coordinate
dependence of the PA estimator rather than ambiguity in algebraic or covariance
rank. Increasing the permutation budget estimates each coordinate-specific
reference distribution more accurately but does not make those distributions
equal.

The experimental scorer additionally standardizes each coordinate, enforces a
minimum component count of one, and converts the selected subspace into a
context--answer projection score. Thus reparameterization can propagate through
standardization, component selection, projection, and fixed-threshold
classification. Appendix~\ref{app:score-definitions} gives the complete score,
including numerical safeguards and aggregation rules. The empirical analysis
tests this full pipeline rather than only the centered-covariance setting of
Lemma~\ref{lem:orientation}.

\subsection{Limits of Marginal-Preserving Reference Distributions}
\label{sec:single-view-limit}

The PA calculation raises a broader question. Can any reference procedure
preserve the empirical distribution of every coordinate, remove
cross-coordinate covariance, and still behave consistently under orthogonal
changes of basis? This feasibility question is answered for procedures constructed solely
from the observed representation matrix.

For a row matrix \(X\), define its empirical row measure and projected
empirical measures by
\[
  \widehat\mu_X=\frac{1}{n}\sum_{i=1}^n\delta_{x_i},
  \qquad
  \widehat\mu_{Xv}
  =
  \frac{1}{n}\sum_{i=1}^n\delta_{x_i^\top v}.
\]
Let \(\mathcal K(X,\cdot)\) denote any randomized data-internal reference
procedure that uses only \(X\) to produce an \(n\times d\) reference matrix,
and write \(Z\sim\mathcal K(X,\cdot)\). This restriction distinguishes the
procedure from methods that use paired data or an external control sample.

The procedure preserves coordinate marginals exactly if, for every
fixed \(X\),
\[
  \widehat\mu_{Ze_j}=\widehat\mu_{Xe_j}
  \quad\text{almost surely for }j=1,\ldots,d.
\]
It is right-orthogonally equivariant if, for
every fixed \(X\) and \(R\in O(d)\),
\[
  \mathcal K(XR,\cdot)
  =
  \mathcal L(ZR\mid Z\sim\mathcal K(X,\cdot)).
\]

At first glance, exact marginal preservation and orthogonal equivariance may
appear to constrain only one-dimensional projections. The next theorem shows
that together they determine the entire empirical row distribution.

\begin{theorem}[Marginal-Preservation Limit]
\label{thm:single-view-limit}
Suppose a randomized data-internal reference procedure preserves every
coordinate marginal. Suppose also that, for every input matrix and every
\(R\in O(d)\), its output distribution transforms from \(Z\) to \(ZR\).
Then
\[
  \widehat\mu_Z=\widehat\mu_X
  \quad\text{almost surely}.
\]
Equivalently, \(Z=PX\) for a random row-permutation matrix \(P=P(X,Z)\).
Therefore,
\[
  Z^\top Z=X^\top X
\]
almost surely.
\end{theorem}

The proof is provided in Appendix~\ref{app:single-view-limit}. Its implication is
stronger than covariance preservation alone. An equivariant procedure that
preserves every coordinate marginal must reproduce the complete empirical row
measure, up to row permutation. It therefore cannot serve as a
dependence-destroying reference distribution.

\begin{independentcorollary}[Data-Internal Reference Limitation]
\label{cor:no-null}
A reference procedure constructed solely from \(X\) cannot satisfy all three
requirements simultaneously: exact preservation of every coordinate's
empirical marginal distribution, right-orthogonal equivariance, and alteration
of the covariance matrix \(X^\top X\). In particular, cross-coordinate
covariances cannot be removed while the first two requirements are retained.
\end{independentcorollary}

To clarify the theorem's practical implications,
Table~\ref{tab:reference-assumptions} makes the available trade-offs explicit.
Column permutation obtains a changed covariance reference by giving
up orthogonal equivariance. A trace-matched isotropic Gaussian reference can
be equivariant and change covariance, but it gives up exact empirical marginal
preservation. A common row permutation satisfies both invariance requirements
only because it leaves the covariance unchanged.

\begin{table}[htbp]
\caption{Assumptions and consequences for representative data-internal
reference procedures. ``Different covariance'' asks whether the procedure can
produce \(Z^\top Z\neq X^\top X\).}
\label{tab:reference-assumptions}
\centering
\small
\begin{tabular}{lccc}
\toprule
Reference procedure & Exact marginals & \(O(d)\)-equivariant
& Different covariance\\
\midrule
Independent column permutations & Yes & No & Yes\\
Common row permutation & Yes & Yes & No\\
Trace-matched isotropic Gaussian & No & Yes & Yes\\
\bottomrule
\end{tabular}
\end{table}

Two qualifications limit the scope of this impossibility result. First, the theorem concerns data-internal
procedures with exact coordinate-marginal preservation; it does not exclude
equivariant procedures that relax exact marginal preservation or use paired data or
external information. Second, its proof requires the full group \(O(d)\),
whose rotations can map a fixed coordinate direction to any point on the unit
sphere. It therefore applies directly to the RMSNorm reparameterizations in
this study, but not to \(\mathcal G_{\mathrm{LN}}\), which fixes the all-ones
direction.

The results cover the two settings in complementary ways. The general theorem
governs the full orthogonal class used for the RMSNorm models, while
Theorem~\ref{thm:same-spectrum} establishes PA instability directly within the
LayerNorm-compatible subgroup. The analysis now turns from the impossibility result to scores that
are invariant by construction.

\subsection{Invariant Context--Answer Scores}
\label{sec:invariant-scores}

The comparators are constructed using row Gram matrices, a standard way to
remove dependence on feature coordinates in representational similarity
analysis and centered kernel alignment
\citep{kriegeskorte2008rsa,kornblith2019cka}. In particular, for
\(X,Y\in\R^{n\times d}\),
\[
  XX^\top=YY^\top
  \quad\Longleftrightarrow\quad
  Y=XR
\]
for some \(R\in O(d)\). Thus \(XX^\top\) characterizes \(X\) up to a
right-orthogonal transformation. The comparator scores specialize this
established Gram-matrix construction to the relation between retrieved context
and a target answer within one model. This differs from CKA and RSA, which
compare learned representations.

For the paired object \((X,A)\), applying the same fact to the
vertically concatenated matrix gives the block Gram matrix
\[
  \begin{bmatrix}X\\A\end{bmatrix}
  \begin{bmatrix}X\\A\end{bmatrix}^{\!\top}
  =
  \begin{bmatrix}
    XX^\top & XA^\top\\
    AX^\top & AA^\top
  \end{bmatrix}.
\]
The comparator scores use \(XX^\top\), \(AX^\top\), and answer norms, all of
which are blocks or diagonal entries of the joint Gram matrix. The joint Gram-matrix construction makes their invariance explicit before their empirical discrimination is assessed.

The first comparator retains the smallest principal subspace explaining a
fixed fraction of context variance and measures the answer projection onto that
subspace. The second replaces the hard cutoff with ridge shrinkage over all
context directions. Both can be computed from the context Gram matrix,
context--answer inner products, and answer norms. Their full definitions are
provided in Appendix~\ref{app:score-definitions}. These constants are fixed design choices rather than theoretically privileged
values. Specifically, \(\eta=0.9\),
\(\rho_\ell=0.1\bar\lambda_\ell\), and top-quartile token aggregation were
specified before evaluation and were not tuned to the reported results.

By construction, the resulting scores are jointly right-orthogonally
invariant. Specifically,
for every \(R\in O(d)\), every finite collection of layers, and every fixed
input pair,
\[
  s_{\mathrm{variance}}(\{X_\ell R,A_\ell R\}_{\ell\in\mathcal L})
  =
  s_{\mathrm{variance}}(\{X_\ell,A_\ell\}_{\ell\in\mathcal L}),
\]
and
\[
  s_{\mathrm{ridge}}(\{X_\ell R,A_\ell R\}_{\ell\in\mathcal L})
  =
  s_{\mathrm{ridge}}(\{X_\ell,A_\ell\}_{\ell\in\mathcal L}).
\]
Orthogonal invariance follows directly from
\[
  (X_\ell^cR)(X_\ell^cR)^\top
  =
  X_\ell^c(X_\ell^c)^\top,
  \qquad
  (A_\ell^cR)(X_\ell^cR)^\top
  =
  A_\ell^c(X_\ell^c)^\top,
\]
together with preservation of centered answer norms. The baseline correction,
clipping, token aggregation, and cross-layer geometric mean are deterministic,
so they preserve the equality.

Finally, eigenvalue ties are treated explicitly because numerical determinism
alone does not define an invariant spectral projector. The ridge-shrinkage subspace score
remains invariant with repeated eigenvalues. The cumulative-variance
principal-subspace score is invariant under the tie-safe cutoff stated above;
if the cutoff splits a repeated eigenspace, an ordered numerical
eigendecomposition does not provide a basis-independent mathematical
definition.

\section{Empirical Analysis}
\label{sec:empirical}

The empirical analysis tests whether the theoretical coordinate dependence survives the full
measurement pipeline. Specifically, the analysis asks whether reparameterization changes
PA-based measurements, whether randomness or preprocessing can explain the
observed instability, and whether invariant alternatives retain context
discrimination.

\subsection{Study Design}

To follow this propagation from component selection to downstream decisions,
the pre-specified design includes
five models from four architecture families. The models are Qwen2.5-0.5B, Qwen2.5-1.5B, Qwen3-0.6B, SmolLM2-1.7B, and
Pythia-1B
\citep{qwen2024qwen25,qwen2025qwen3,allal2025smollm2,biderman2023pythia}.
Table~\ref{tab:models} states the transformation class used for each model.

\begin{table}[htbp]
\caption{Models and exact residual-stream transformation classes. Parameter
counts are those encoded in the public model identifiers.}
\label{tab:models}
\centering
\begin{tabular}{llll}
\toprule
Model & Architecture & Normalization & Transformation class\\
\midrule
Qwen2.5-0.5B & Qwen2 & RMSNorm & \(O(d)\)\\
Qwen2.5-1.5B & Qwen2 & RMSNorm & \(O(d)\)\\
Qwen3-0.6B & Qwen3 & RMSNorm & \(O(d)\)\\
SmolLM2-1.7B & Llama-style & RMSNorm & \(O(d)\)\\
Pythia-1B & GPT-NeoX & LayerNorm & \(\{R\in O(d):R\1=\1\}\)\\
\bottomrule
\end{tabular}
\end{table}

To vary the form of retrieved evidence, FinanceBench, QASPER, and HotpotQA are used
\citep{islam2023financebench,dasigi2021qasper,yang2018hotpotqa}. FinanceBench
pairs contrast evidence from the target filing with a same-company passage
from a different reporting period; the evidence-ablated condition contains only
the period-mismatched passage. For QASPER, the evidence-containing condition
contains annotated evidence paragraphs plus three sampled non-evidence
paragraphs from the same paper, while the evidence-ablated condition contains
those same non-evidence paragraphs without the evidence. HotpotQA uses the
analogous construction with supporting and distractor documents.

Despite these domain-specific constructions, the paired conditions always
differ by the presence of
annotated evidence, not by the question or evaluated answer. These
conditions are used as operational labels for context sufficiency; they do not
constitute an independent semantic verification that every evidence-containing
context is sufficient.

Within each model--domain combination, 20 question pairs are assigned to
calibration and 50 disjoint pairs to evaluation.
The same target-answer tokens are teacher-forced
under both contexts. Contexts are truncated to 192 tokens. Hidden states are probed at
one-quarter, one-half, and three-quarters of model depth, and the three layer
scores are aggregated with a geometric mean.

Alongside the representation-based scores, answer log-probability is included as a
behavioral baseline. It is the mean teacher-forced log-probability across the
fixed target-answer tokens; larger values indicate the evidence-containing
condition. It uses the same model--domain-specific calibration and evaluation
procedure as the representation-based scores.

To intervene on hidden coordinates while preserving model behavior, five
transformations are generated for each model from fixed random seeds, and each
transformation is reused across that model's three domains. For an RMSNorm
width \(d=mb\), where \(b\) is the largest power of two dividing \(d\), the
implemented matrix is
\[
  R=(Q_m\otimes H_b)\Pi D.
\]
Here \(H_b\) is a normalized Hadamard matrix, \(Q_m\) is obtained by QR
factorization of an \(m\times m\) standard Gaussian matrix with the diagonal
signs of the triangular factor absorbed into \(Q_m\), \(\Pi\) is a uniformly
sampled coordinate permutation, and \(D\) contains independent Rademacher
signs. Thus the RMSNorm intervention is a structured dense orthogonal draw,
not a Haar draw from \(O(d)\). For the studied LayerNorm model, \(Q_\perp\) is sampled Haar-uniformly from
\(O(d-1)\) using the same sign-corrected Gaussian QR construction. The matrix
\(\operatorname{diag}(1,Q_\perp)\) is then embedded in a basis whose first
vector is \(\1/\sqrt d\) and conjugated back. The resulting matrix fixes
\(\1\) and is Haar-uniform on its orthogonal complement. The design contains
25 distinct model--transformation matrices and 75 dependent
model--domain--transformation comparisons.

Before representation stability was measured, numerical verification confirmed that all 25
transformations preserved the evaluated behavior. Greedy output tokens
were unchanged, the maximum absolute teacher-forced log-probability error was
below \(10^{-4}\), and relative logit \(\ell_2\) error was below \(10^{-4}\) on
ten verification prompts. These checks establish numerical preservation of the
evaluated input--output behavior.

Within each transformed or original representation, PA uses 50
column-permutation replicates and examines
\[
  k=\min\{64,\min(n_\ell,d)-1\}
\]
components. The componentwise 0.95 cutoff is estimated using linear interpolation of the
empirical permutation quantiles.
Consecutive-from-largest selection uses strict exceedance and stops at the
first failure; if all \(k\) comparisons succeed, the selected count is \(k\).
The scoring implementation then enforces a minimum count of one. The reported
context component count is the integer median of the three layer-specific
selected counts. In the centering-only control, each original--transformed pair
uses the same item-level permutation seed and the same 50 replicates. Thus its
disagreement cannot be attributed to a seed change between orientations.

After the scores are computed, decision thresholds are selected to maximize
accuracy over the 40
calibration contexts; if multiple cutoffs maximize accuracy, their median is
used. A threshold is selected separately for each score and model--domain
combination on the original parameterization and is not recalibrated after
reparameterization. AUROC is threshold free and is computed on the 100
evaluation contexts in each model--domain combination.

The four primary outcomes comprise context-level selected-component-count
disagreement, absolute score change, fixed-threshold decision disagreement,
and within-question ordering reversal. Because each
question contributes matched evidence-containing and evidence-ablated
contexts, questions rather than individual contexts are used as the resampling
units. For the main stability analysis, uncertainty calculations condition on the
five fixed transformations and use the question as the independent resampling
unit; they quantify variation over questions, not over arbitrary
transformations. Paired resampling retains the matched evidence-containing and
evidence-ablated contexts. The control and discrimination intervals are
descriptive empirical-combination bootstraps over the 12 and 15 observed
model--domain combinations, respectively; because combinations share models
and domains, they are not population-sampling intervals. The robustness analysis uses a question-level bootstrap within its four
representative settings.

\subsection{Reparameterization-Induced Instability}

With the design fixed, instability is first examined under
function-preserving reparameterization.
Table~\ref{tab:stability-results} summarizes the 75
model--domain--transformation comparisons. The most consequential endpoint is
the downstream decision. The median fixed-threshold disagreement rate is 0.26,
whereas an independent-seed rerun without reparameterization changes none of
the 1,200 decisions in the matched control. Pairwise ordering is more stable
than absolute decisions but is not invariant, with a median reversal rate of
0.10. At the component-selection stage, the median comparison has a
context-level disagreement rate of 0.79. Because this binary rate does not
describe magnitude, the 7,500 matched context transformations are also pooled.
The median absolute component-count change is 2
(IQR 1--4; range 0--12). Conditional on a change, the median is 3 (IQR 1--5).
Original component counts have median 39 (IQR 35--43; range 11--53), whereas
reparameterized counts have median 37 (IQR 35--39; range 9--47).

Beyond the absolute magnitude, the signed shift is directional: 5,025 transformations decrease the count,
1,238 increase it, and 1,237 leave it unchanged. The directional asymmetry should be interpreted relative to the intervention. The RMSNorm transformations are
structured randomized Hadamard--block rotations rather than Haar draws from
\(O(d)\); they tend to redistribute coordinate concentration, and the observed
39-to-37 median shift and narrower transformed IQR are consistent with movement
toward the constant-diagonal, diffuse endpoint in
Theorem~\ref{thm:spectral-range}. The theorem's aligned and constant-diagonal
orientations establish extremal reference configurations, not a prediction
that arbitrary Haar transformations must decrease the selected count. Accordingly, the
direction observed here is specific to the sampled transformation family,
whereas disagreement under function preservation is the primary result.

To connect this empirical pattern back to the theory, the spectrum-conditioned lower bound from
Theorem~\ref{thm:spectral-range} is also evaluated at \(\alpha=0.05\) using
the centered, unstandardized covariance \(S=X_c^\top X_c/n\) for the 1,200
original contexts from the four full-\(O(d)\) RMSNorm models. Across 3,600
context--layer spectra, \(m_\alpha(S,n)=1\) for 182 and \(m_\alpha(S,n)=0\)
for 3,418; the 182 positive cases occur in 182 distinct contexts. Thus the
conservative theorem certifies an orbit-wide component-count range of at least
one for 15.2\% of contexts and is inconclusive, rather than contradicted, for
the remaining 84.8\%. The spectral range bound is not applied to Pythia because its
LayerNorm-compatible transformation class is a proper subgroup of \(O(d)\).

\begin{table}[htbp]
\caption{Stability under analytically function-preserving, numerically
verified reparameterizations. The first three rows summarize 75
model--domain--transformation comparison-level rates. The final row pools
7,500 matched context transformations. IQR and range are descriptive, not
confidence limits.}
\label{tab:stability-results}
\centering
\begin{tabular}{lccc}
\toprule
Endpoint & Median & IQR & Range\\
\midrule
PA component-count disagreement & 0.79 & (0.725, 1.000) & (0.59, 1.00)\\
Decision disagreement & 0.26 & (0.160, 0.345) & (0.00, 0.44)\\
Within-question ordering reversal & 0.10 & (0.060, 0.140) & (0.02, 0.24)\\
Absolute component-count change & 2 & (1, 4) & (0, 12)\\
\bottomrule
\end{tabular}
\end{table}

Figure~\ref{fig:stability-heterogeneity} shows that the instability occurs across model--domain combinations rather than being driven by a single setting. Figure~\ref{fig:stability-overview} complements this comparison by showing the distribution, magnitude, and decision-level consequences of the component-count changes.

\begin{figure}[htbp]
\centering
\includegraphics[width=\textwidth]{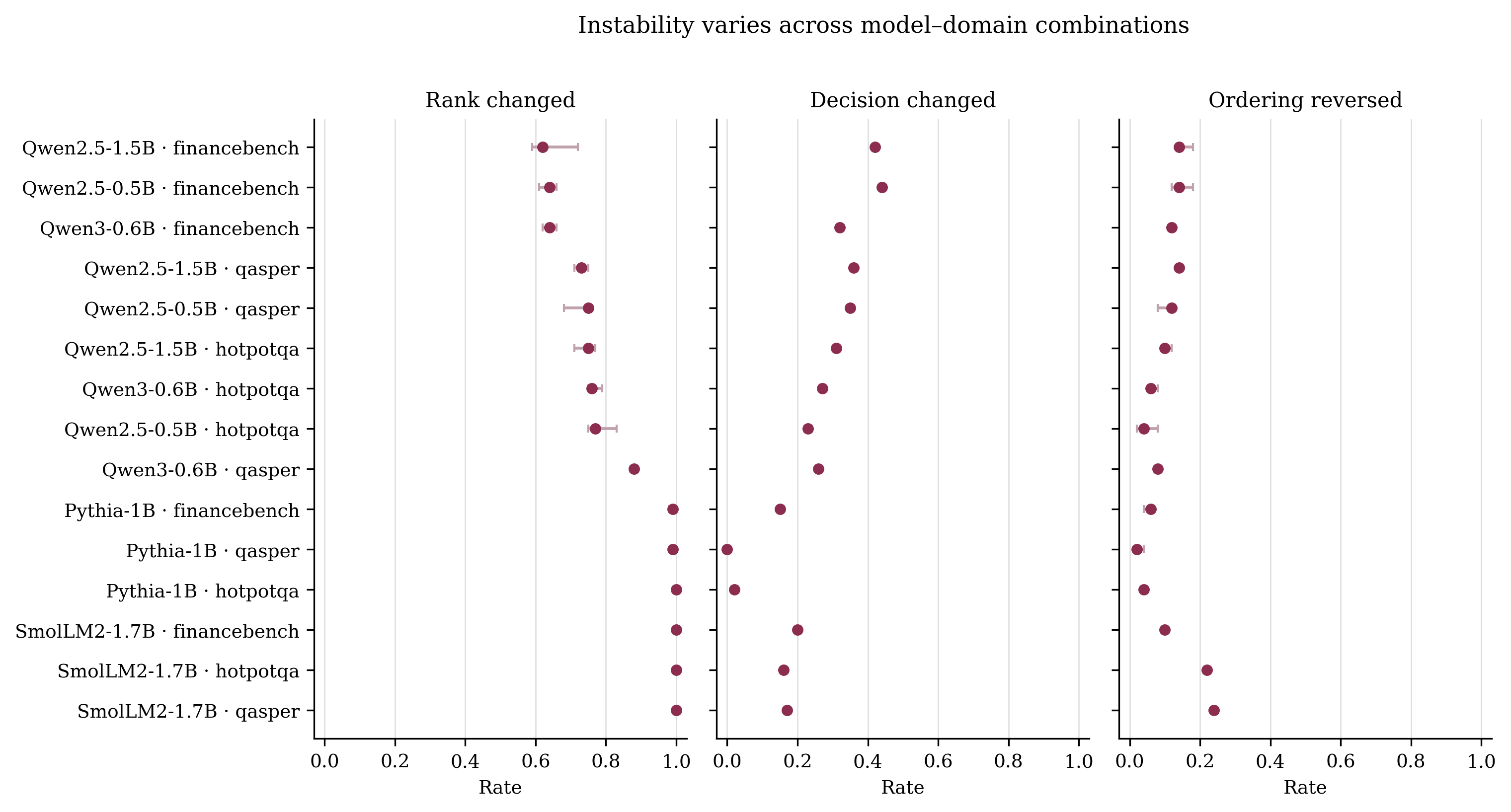}
\caption{Heterogeneity across the 15 model--domain combinations. Each point is
the median across the five fixed transformations; horizontal bars show the
minimum and maximum across those transformations. The ranges describe the five
evaluated transformations and are not confidence intervals for an underlying
transformation population.}
\label{fig:stability-heterogeneity}
\end{figure}

\begin{figure}[htbp]
\centering
\includegraphics[width=\textwidth]{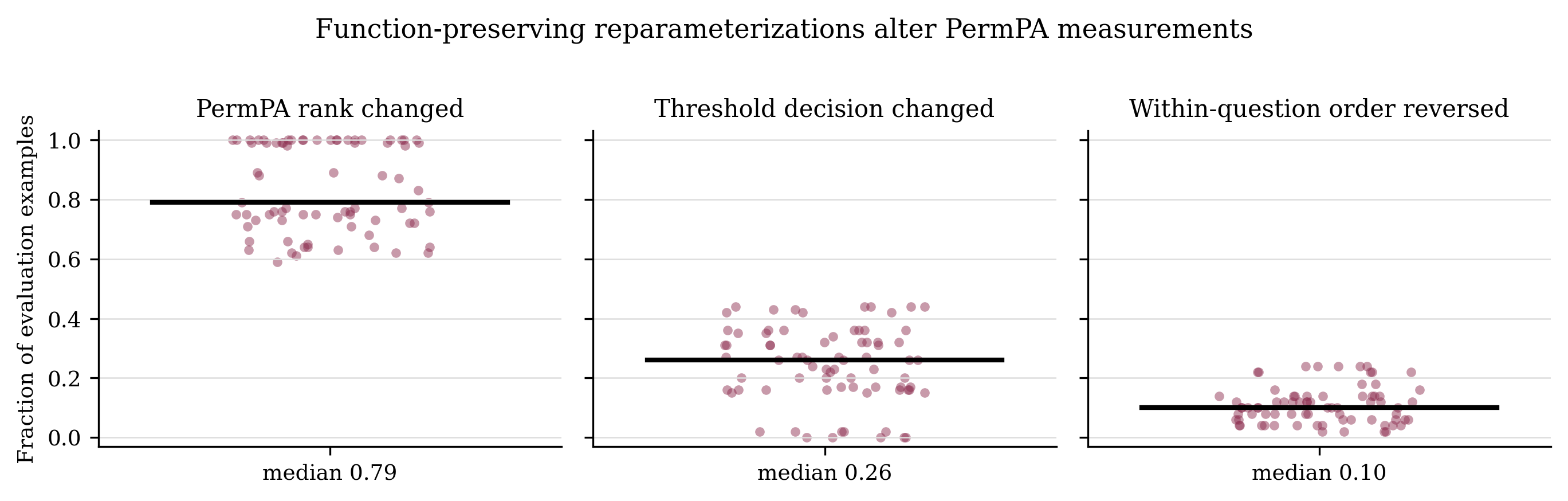}
\caption{Distribution and downstream consequences of component-count changes.
Top left. raw selected component counts for 1,500 original contexts
and 7,500 matched reparameterized contexts. Top right. absolute
changes for the 7,500 matched context transformations. Bottom.
decision disagreement and within-question ordering reversal for each of the 75
model--domain--transformation comparisons; horizontal segments mark medians.
Decisions use the score threshold selected on the original calibration split.}
\label{fig:stability-overview}
\end{figure}

As expected from their construction, the maximum absolute changes of the
cumulative-variance principal-subspace and ridge-shrinkage scores are
\(5.41\times10^{-8}\) and \(4.68\times10^{-8}\), respectively, which are of
numerical-error magnitude.

The PA-selected subspace score has the larger question-level absolute change
than either invariant comparator for all 50 questions in every one of the 75
comparisons. This is a descriptive finite-sample statement; no hypothesis test
is needed for the algebraically invariant comparators.

Importantly, no activation component count is treated as ground truth.
Instead, these comparisons test the more basic requirement that the same input--output function
receive the same measurement.

\subsection{Randomness and Preprocessing Checks}\label{sec:controls}

The next analysis asks whether randomness or preprocessing can account for the observed
effect. Because PA is randomized, a selected component count can change when the
permutation reference distribution is recomputed even if the activations are
held fixed. This baseline variability is quantified in a post-hoc control that covers the 12
model--domain combinations from the four RMSNorm models. For each combination,
PA is recomputed on the original, untransformed activations using an
independent permutation seed and the same 50 permutation-reference replicates.
The comparison with reparameterization is paired by model and domain; the
reparameterization rate is the median across the five fixed transformations.

The comparison in Table~\ref{tab:controls} shows that random-seed variability
is much smaller than
the coordinate effect. The median no-reparameterization seed-change rate is
0.06, whereas the corresponding reparameterization rate is 0.76. The median
paired difference is 0.69, with a model--domain bootstrap 95\% interval of
(0.625,0.880). Reparameterization has the larger rate in all 12
model--domain combinations (two-sided sign-test \(p=4.88\times10^{-4}\)).
Across the 1,200 seed-control contexts, the absolute component-count change has
median 0 and IQR 0--0; only 74 contexts have a nonzero change. Changing only the
permutation seed produces no fixed-threshold decision disagreement in any of
the 1,200 contexts. By comparison, the median reparameterization-induced
decision-disagreement rate across the same 12 model--domain combinations is
0.29.
In this 12-combination, one-rerun control, finite-replicate Monte Carlo
variation is therefore substantially smaller than the observed
reparameterization effect.

\begin{table}[htbp]
\caption{Post-hoc controls across 12 model--domain combinations. Rates are
fractions of 100 evaluation contexts. The reparameterization rate is summarized
over five transformations for the seed comparison. The seed control uses one
independent PA rerun. The preprocessing comparison uses the same single
transformation (seed 101) and permutation seed in its standardized and
centering-only arms.}
\label{tab:controls}
\centering
\begin{tabular}{lc}
\toprule
Quantity & Median\\
\midrule
Reparameterization component-count disagreement & 0.76\\
No-reparameterization independent-seed rate & 0.06\\
Seed-only absolute component-count change & 0 (IQR 0--0)\\
Paired reparameterization-minus-seed difference & 0.69\\
Reparameterization decision disagreement & 0.29\\
No-reparameterization seed-only decision disagreement & 0.00\\
Centering-only reparameterization rate & 1.00\\
Standardized-minus-centering-only difference & \(-0.24\)\\
\bottomrule
\end{tabular}
\end{table}

Figure~\ref{fig:seed-control} visualizes the separation between the reparameterization effect and independent-seed variability for both component counts and downstream decisions.

\begin{figure}[htbp]
\centering
\includegraphics[width=\textwidth]{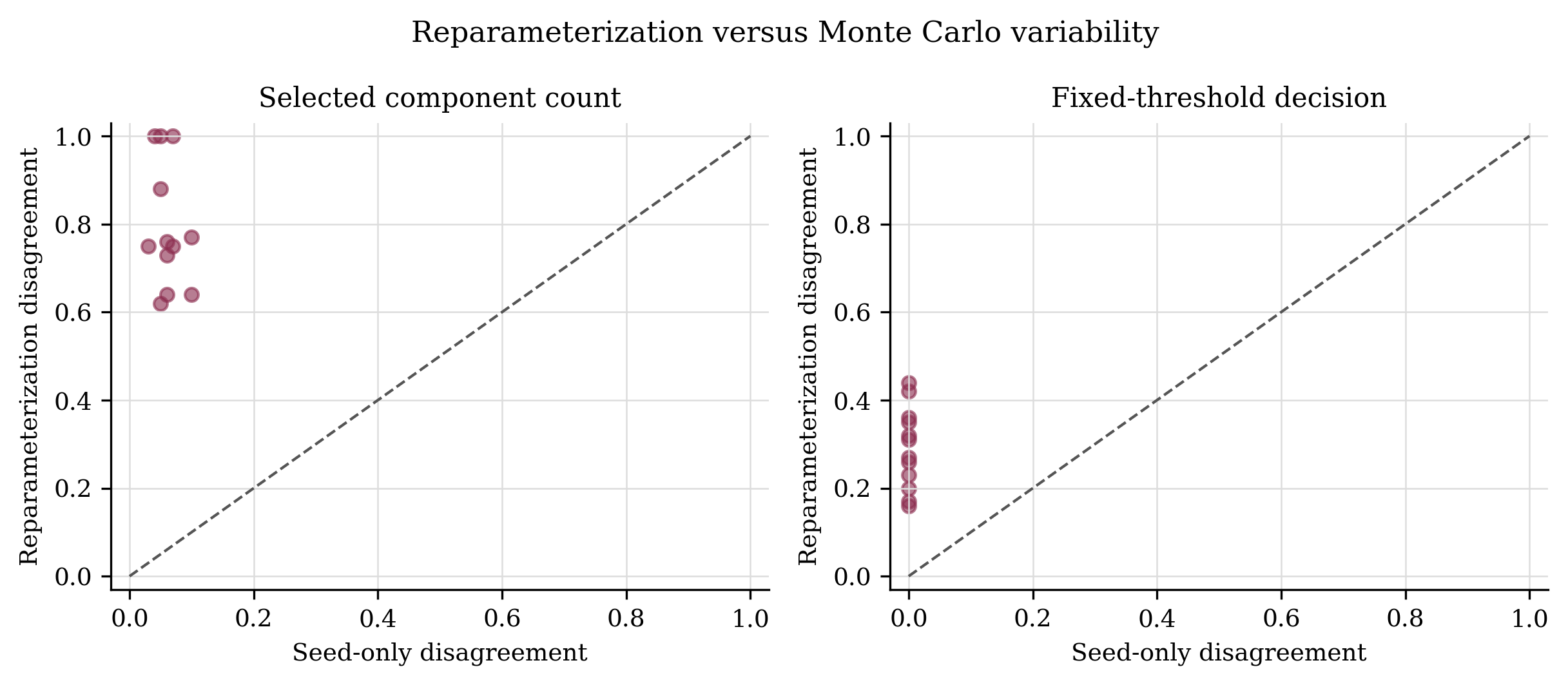}
\caption{Reparameterization-induced versus independent-seed
no-reparameterization disagreement for selected component counts
(left.) and fixed-threshold decisions (right.). Each point is
one of 12 model--domain combinations; the dashed line denotes equality.}
\label{fig:seed-control}
\end{figure}

To isolate the reference-distribution mechanism from preprocessing, the analysis turns
to the centering-only arm, which provides the cleanest empirical test of
Lemma~\ref{lem:orientation} and the PA component-count results in
Theorems~\ref{thm:spectral-range} and \ref{thm:same-spectrum}. It removes
coordinatewise standardization and thereby separates the observed and
reference spectra.
If \(X_c\) denotes the centered context matrix, then the reparameterized
centered matrix is exactly \(X_cR\), and
\[
  (X_cR)^\top(X_cR)=R^\top X_c^\top X_cR.
\]
It follows that the observed eigenvalues are identical in exact arithmetic. Any
change in the selected component count must arise from the
column-permutation reference distribution, apart from numerical error. Direct
numerical comparison across contexts and layers gives maximum relative
observed-eigenvalue error \(1.26\times10^{-6}\) (maximum absolute error
\(4.69\times10^{-2}\), on the unnormalized eigenvalue scale). Despite
the unchanged observed spectrum, 1,141 of 1,200 contexts change selected
component count; the median model--domain disagreement rate is 1.00 (range
0.78--1.00). The corresponding standardized-minus-centering-only median
difference is \(-0.24\) (model--domain bootstrap 95\% interval
\([-0.310,0.025]\)). Coordinatewise standardization is therefore not required
for the instability.

Having isolated the reference distribution under centering alone, the analysis uses the
standardized arm to test a different mechanism. Let
\(\widetilde X=X_cD_X^{-1}\) use the implementation's \(n-1\)-denominator
sample standard deviations and let
\(\widetilde S=\widetilde X^\top\widetilde X/n\). Every diagonal entry then
equals \(c_n=(n-1)/n\), so Lemma~\ref{lem:orientation} gives
\[
  \E_\pi[\widetilde S^\pi]=c_n I_d,
  \qquad
  V(\widetilde S)=d(d-1)c_n^2
\]
in every coordinate system. With \(1/n\)-denominator standard deviations,
these expressions reduce to \(I_d\) and \(d^2-d\), respectively.
Standardization therefore suppresses the first-order,
diagonal-of-covariance mechanism isolated by centering alone. In the
standardized arm, residual instability may reflect a changed observed spectrum
because standardization is not orthogonally equivariant, coordinate-dependent
higher moments of the finite permutation reference, or both. The current
comparison establishes their combined effect but does not attribute it to one
source. The two arms are nevertheless complementary: centering alone directly
confirms reference-driven instability with a fixed observed spectrum, while
the standardized arm shows that the complete implemented pipeline remains
unstable after coordinate variances are equalized.

Figure~\ref{fig:reference-isolation} visualizes this mechanism. Disagreement remains high in the centering-only arm even though the observed covariance spectrum is unchanged.

\begin{figure}[htbp]
\centering
\includegraphics[width=0.58\textwidth]{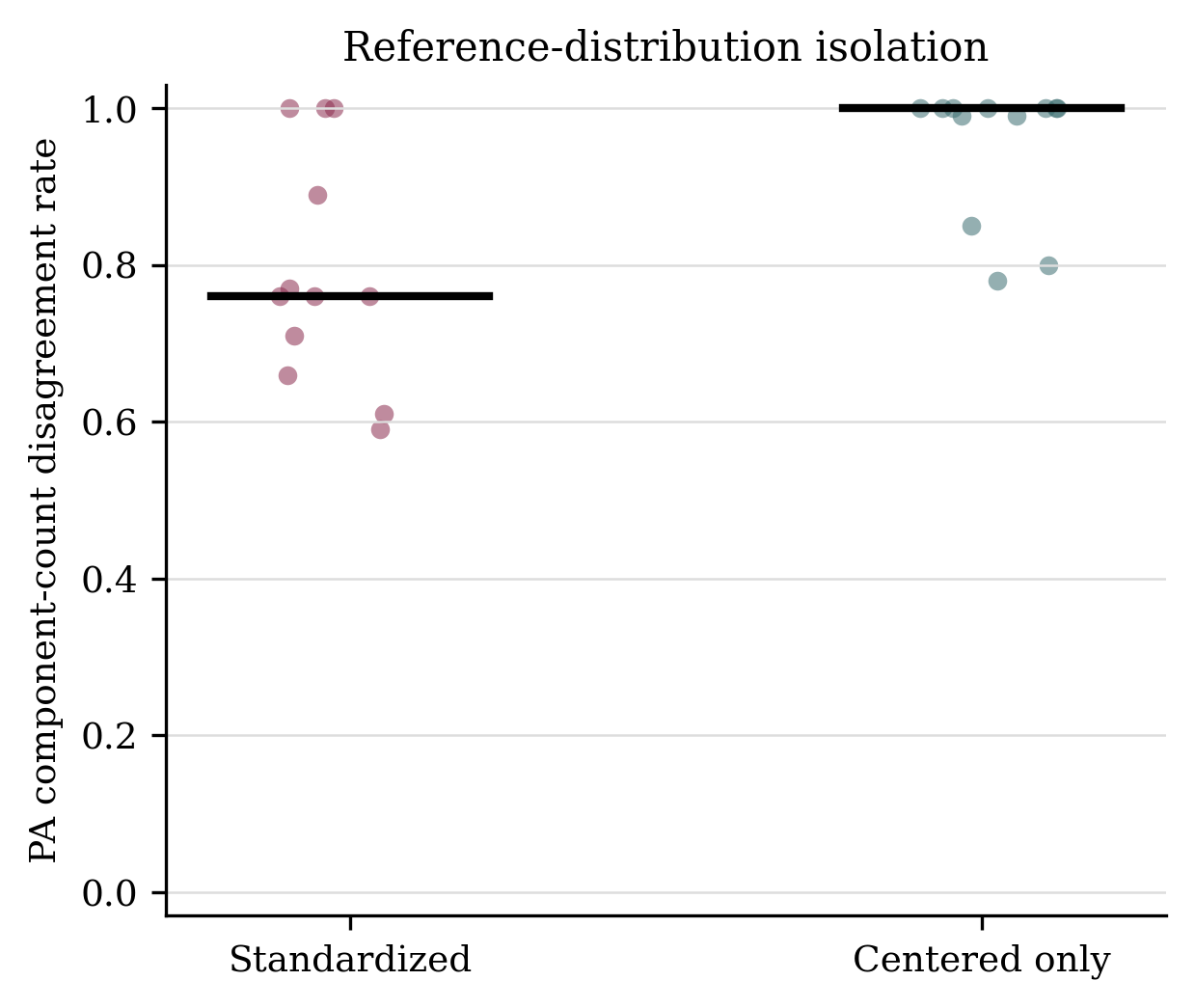}
\caption{Reference-distribution isolation across 12 model--domain
combinations. Under centering alone, the observed covariance spectrum is
exactly invariant to the orthogonal reparameterization, while the
column-permutation reference changes. Points are model--domain disagreement
rates under transformation seed 101; horizontal segments mark medians.}
\label{fig:reference-isolation}
\end{figure}

\subsection{Context Discrimination}

All four scores, including answer log-probability, have median held-out AUROC
within a 0.023 band across the 15 model--domain combinations in
Table~\ref{tab:utility-results}. The ridge-shrinkage subspace score has the
highest median AUROC, 0.712, compared with 0.703 for the PA-selected subspace
score and 0.689 for answer log-probability. The difference between the two
representation-score marginal medians is 0.009.
Across the 15 matched model--domain combinations, the ridge-shrinkage score is
higher in 11, and the median paired difference is 0.028, with a model--domain
bootstrap 95\% interval of
\([-0.003,0.051]\) and a range from \(-0.010\) to 0.080. No score dominates across all model--domain combinations, and the
representation-based scores do not consistently outperform answer
log-probability. The relevant result is therefore not superiority: the tested
right-orthogonally invariant scores remove reparameterization sensitivity
without an observed loss in median discrimination.

\begin{table}[htbp]
\caption{Held-out context discrimination across 15 model--domain
combinations. Ranges show the minimum and maximum AUROC across combinations.}
\label{tab:utility-results}
\centering
\begin{tabular}{lcc}
\toprule
Score & Median AUROC & Range\\
\midrule
PA-selected subspace score & 0.703 & (0.637, 0.813)\\
Cumulative-variance principal-subspace score & 0.697 & (0.664, 0.824)\\
Ridge-shrinkage subspace score & 0.712 & (0.683, 0.827)\\
Answer log-probability & 0.689 & (0.527, 0.839)\\
\bottomrule
\end{tabular}
\end{table}

Figure~\ref{fig:utility-stability} summarizes the central trade-off. The invariant scores remain numerically stable while retaining discrimination comparable to the PA-selected score.

\begin{figure}[htbp]
\centering
\includegraphics[width=0.94\textwidth]{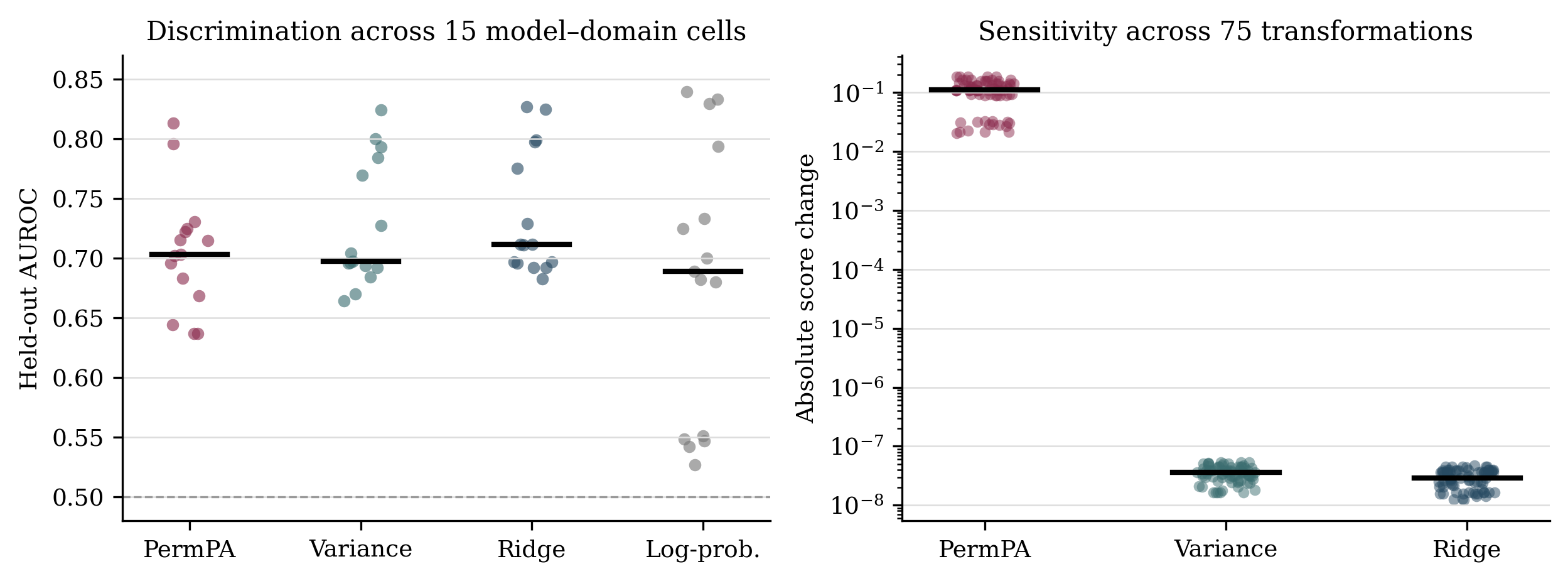}
\caption{Left. held-out discrimination. Each point is one of 15
model--domain combinations, and each horizontal segment is the median.
Right. sensitivity to function-preserving reparameterization. Each
point is one of 75 model--domain--transformation comparisons. PA-selected
subspace values are median absolute score changes across evaluation contexts;
the right-orthogonally invariant comparator-score values are
maximum absolute changes and therefore provide a conservative numerical
stability check.}
\label{fig:utility-stability}
\end{figure}

\subsection{Robustness Analysis}
\label{sec:robustness}

The randomness and preprocessing checks above assess alternative explanations
for the main effect. The robustness analysis instead tests whether the
instability persists across PA settings used in the primary analysis. This
post-hoc sensitivity analysis uses one representative
model--domain combination from each architecture family and one fixed
function-preserving transformation per combination. Each combination contains
the same 100 held-out contexts as the main analysis. The analysis varies the number of
permutation replicates (\(50,100,200\)), the permutation-reference quantile
(\(0.90,0.95,0.99\)), the selection rule (consecutive-from-largest selection
or counting all components exceeding their permutation-reference thresholds),
and whether the selected component count is measured at each probed layer or
by the median across layers. For efficiency and direct comparability, settings
with fewer than 200 replicates use prefixes of the same set of 200
permutation-reference replicates. The complete
\(3\times3\times2\times4\) factorial grid contains 72 settings. All 72
settings reuse the same single fixed transformation within each representative
model--domain combination; the sweep varies analysis choices, not the
transformation draw.

The conclusion remains stable across the tested choices. At the original
setting (50 replicates, the 0.95 quantile, consecutive-from-largest selection,
and median-across-layer aggregation), the median disagreement rate across the
four representative combinations is 0.87 (question-bootstrap 95\% interval
\([0.825,0.910]\)). Across the complete 72-setting grid, the median remains
between 0.86 and 0.99. Increasing the permutation budget or changing the
quantile, selection convention, or aggregation level does not materially
reduce the observed instability. Table~\ref{tab:robustness} reports
baseline-centered one-factor slices, and Figure~\ref{fig:robustness} shows
between-setting heterogeneity.

\begin{table}[htbp]
\caption{Baseline-centered slices of the full factorial sensitivity analysis
in four representative model--domain settings. Each row varies one choice
while holding the others at 50 replicates, the 0.95 permutation-reference
quantile, consecutive-from-largest selection, and median-across-layer
aggregation. Intervals are question-level bootstrap 95\% intervals for the
median across four representative model--domain combinations.}
\label{tab:robustness}
\centering
\small
\begin{tabular}{@{}lcc@{}}
\toprule
Choice & Values & Median rate\\
\midrule
Permutation replicates & 50 / 100 / 200 & 0.870 / 0.865 / 0.865\\
Reference quantile & 0.90 / 0.95 / 0.99 & 0.880 / 0.870 / 0.885\\
Selection rule & Consecutive from largest / count all & 0.870 / 0.870\\
Layer aggregation & Quarter / half / three-quarter / median
  & 0.925 / 0.980 / 0.935 / 0.870\\
\bottomrule
\end{tabular}
\end{table}

\begin{figure}[htbp]
\centering
\includegraphics[width=\textwidth]{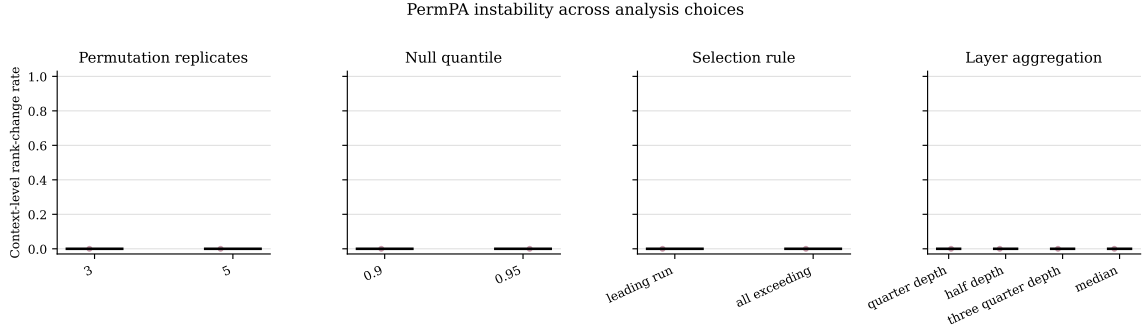}
\caption{Context-level PA component-count disagreement rates across analysis
choices. Points are the four representative
model--domain--transformation settings; horizontal segments mark medians. In
each panel, non-displayed choices are held at the main experimental setting.}
\label{fig:robustness}
\end{figure}

\section{Discussion and Limitations}
\label{sec:discussion}

The results support a measurement-validity interpretation rather than a claim
that PA lacks descriptive value. Column-permutation PA can distinguish the two
context conditions, but its selected component counts and downstream decisions
also depend on hidden coordinates that are not fixed by the model's
input--output function. The invariant comparators remove this sensitivity while
retaining similar observed discrimination. All four scores lie within 0.023
median AUROC, and the representation-based scores do not consistently
outperform answer log-probability. The comparators therefore serve as
methodological baselines rather than proposed optimal retrieval scores.

The empirical scope covers five small open-weight models, three
question-answering domains, and the tested architecture-compatible
transformations. The randomness and preprocessing checks use the four RMSNorm
models. The seed control consists of one independent rerun and therefore does
not estimate a distribution over permutation seeds, while the standardization
ablation uses one fixed transformation (seed 101). The robustness grid uses one
fixed transformation for each representative model--domain combination. These
choices provide controlled evidence for the proposed mechanism but do not
exhaust larger models, other tasks, or all function-preserving transformations.
The general marginal-preservation limitation requires the full orthogonal group
\(O(d)\) and has not been extended to the mean-preserving subgroup used for
LayerNorm; the LayerNorm result is instead established by the separate PA
construction in Theorem~\ref{thm:same-spectrum}. Future work can extend the
analysis to broader model families and develop invariant measures with stronger
predictive utility.

\section{Conclusion}
\label{sec:conclusion}

This study shows that column-permutation PA is not invariant under
function-preserving reparameterization. Although model behavior, total
variance, and the observed covariance spectrum remain unchanged, the
coordinate-dependent permutation reference can alter selected component
counts, projection scores, and fixed-threshold decisions. The theoretical
results establish this conflict for the full orthogonal class used with
RMSNorm and through a separate construction for the LayerNorm-compatible
subgroup.

The experiments demonstrate that the effect is larger than the measured seed
variability and persists across the tested PA settings. The median
fixed-threshold decision disagreement rate is 0.26 under reparameterization,
compared with 0.00 in the independent-seed control. Under centering alone,
1,141 of 1,200 component counts change despite an unchanged observed spectrum.
Right-orthogonally invariant comparators remove the tested sensitivity while
retaining similar observed discrimination, although all four scores lie within
0.023 median AUROC and the representation-based scores do not consistently
outperform answer log-probability. Representation-derived statistics should
therefore be evaluated under architecture-compatible, function-preserving
reparameterizations before being interpreted as properties of a model, its
evidence, or its decisions.

\section{Code and Data Availability}
\label{sec:reproducibility}

The protocols, analysis code, intermediate results, and figure-generation
scripts are publicly available at github and can be found here\footnote{\url{https://github.com/abdullahkarasan/representation-measurement-invariance}}. The main experimental
specification was fixed before the reported analyses and is recorded in
\texttt{publication\_protocol.json}; the post-hoc controls and sensitivity
grid are specified separately in \texttt{missing\_controls\_protocol.json}
and \texttt{robustness\_protocol.json}, so that pre-specified and exploratory
analyses remain distinguishable. These files document the models, data splits,
token limit, probed layers, PA settings, and fixed random seeds used in each
analysis. The accompanying scripts reproduce the function-preserving
transformations, their numerical verification, the main experiments, the
control analyses, the spectral-bound evaluation, the robustness analysis, and
all figures. Context-level outputs and aggregate summaries are included, so the
reported tables and figures can be regenerated without repeating model
inference.

\acks{No external funding supported this work. The author declares no
competing interests.}

\appendix
\section{Complete Score Definitions}
\label{app:score-definitions}

This appendix records the complete implemented definitions omitted from the
main theoretical presentation.

\subsection{PA-Selected Subspace Score}

The scorer used in the experiments additionally standardizes coordinates,
enforces a minimum component count of one, and converts the selected subspace
into a context--answer score. The complete pipeline is therefore defined before
testing whether the theoretical mechanism survives these implementation
choices.

Specifically, the implemented PA-selected subspace scorer first maps
\[
  X\longmapsto \mathcal S(X)
  =
  (X-\1\bar x^\top)D_X^{-1},
\]
where \(D_X\) contains coordinatewise sample standard deviations, each floored
at \(10^{-6}\). In general,
\[
  \mathcal S(XR)\neq\mathcal S(X)R.
\]
Consequently, under reparameterization, both the standardized observed
spectrum and its permutation reference may change. The empirical study
therefore tests stability of the complete scoring pipeline, not only the raw
covariance version in Lemma~\ref{lem:orientation}.

To connect component selection to the empirical task, define the complete
PA-selected subspace score as follows. At layer \(\ell\), let
\[
  \widetilde X_\ell
  =
  (X_\ell-\1\bar x_\ell^\top)D_{X_\ell}^{-1},
  \qquad
  \widetilde A_\ell
  =
  (A_\ell-\1\bar x_\ell^\top)D_{X_\ell}^{-1}.
\]
Let \(r_\ell^{\mathrm{PA}}\geq1\) be the PA-selected component count after the
implemented minimum-count rule, and let
\(V_\ell^{\mathrm{PA}}\in\R^{d\times r_\ell^{\mathrm{PA}}}\) contain the
corresponding leading right singular vectors of \(\widetilde X_\ell\). For
target-answer token \(i\), define the projected squared-norm fraction
\[
  q_{\ell i}^{\mathrm{PA}}
  =
  \frac{\lVert\widetilde a_{\ell i}V_\ell^{\mathrm{PA}}\rVert_2^2}
       {\lVert\widetilde a_{\ell i}\rVert_2^2}.
\]
The implementation floors the denominator at \(10^{-12}\), assigning zero to
a zero answer vector.
Writing \(\Agg_i\) for the mean of the largest
\(\max\{1,\lceil m_\ell/4\rceil\}\) token-level values, the layer score is
\[
  e_\ell^{\mathrm{PA}}
  =
  \clip_{[0,1]}
  \left(
    \frac{\Agg_i(q_{\ell i}^{\mathrm{PA}})
          -r_\ell^{\mathrm{PA}}/d}
         {1-r_\ell^{\mathrm{PA}}/d}
  \right).
\]
If \(r_\ell^{\mathrm{PA}}=d\), the score \(e_\ell^{\mathrm{PA}}\) is set to one. The
reported PA-selected subspace score is
\[
  s_{\mathrm{PA}}
  =
  \exp\left\{
    \frac{1}{|\mathcal L|}
    \sum_{\ell\in\mathcal L}
    \log\!\bigl(\clip_{[10^{-6},1]}(e_\ell^{\mathrm{PA}})\bigr)
  \right\}.
\]
The definition of the PA-selected subspace score separates the PA-selected component count from the downstream
projection score whose stability and discrimination are evaluated. It also
exposes the propagation path tested empirically. Reparameterization may change
standardization, component selection, the projection score, and the resulting
fixed-threshold classification.

\subsection{Right-Orthogonally Invariant Comparator Scores}

At each layer \(\ell\), both representation sets are centered using the context
mean:
\[
  X_\ell^c=X_\ell-\1_{n_\ell}\bar x_\ell^\top,
  \qquad
  A_\ell^c=A_\ell-\1_{m_\ell}\bar x_\ell^\top.
\]
Let
\[
  G_\ell
  =
  \frac{1}{n_\ell}X_\ell^c(X_\ell^c)^\top
  =
  U_\ell\Lambda_\ell U_\ell^\top,
\]
where
\(\Lambda_\ell=\diag(\lambda_{\ell1},\ldots,\lambda_{\ell p_\ell})\)
contains the eigenvalues in decreasing order that exceed
\[
  \epsilon_{\mathrm{eig},\ell}
  =
  \epsilon_{\mathrm{mach}}\max\{n_\ell,d\}
  \max\{1,\lambda_{\ell1}\}.
\]
If \(p_\ell=0\), both comparator layer scores are defined as zero. Otherwise,
define
\[
  Z_\ell
  =
  A_\ell^c(X_\ell^c)^\top U_\ell
  (n_\ell\Lambda_\ell)^{-1/2}.
\]
This is \(A_\ell^cV_\ell\), written using only the row Gram matrix and
context--answer inner products.

Let \(r_{\ell,\eta}\) be the smallest \(r\) such that
\[
  \frac{\sum_{j=1}^r\lambda_{\ell j}}
       {\sum_{j=1}^{p_\ell}\lambda_{\ell j}}
  \geq\eta.
\]
For the cumulative-variance principal-subspace score, the theoretical
invariance statement assumes that the cutoff does not split a repeated
eigenvalue:
\[
  \lambda_{\ell,r_{\ell,\eta}}>
  \lambda_{\ell,r_{\ell,\eta}+1},
\]
where \(\lambda_{\ell,p_\ell+1}=0\). This condition makes the selected spectral
projector uniquely defined; the numerical implementation uses the ordered
eigendecomposition at the specified tolerance.
For answer row \(a_{\ell i}^c\), define
\[
  h_{\ell i}
  =
  \frac{\sum_{j=1}^{r_{\ell,\eta}}Z_{\ell,ij}^2}
       {\lVert a_{\ell i}^c\rVert_2^2},
  \qquad
  g_{\ell i}
  =
  \frac{\sum_{j=1}^{p_\ell}
  \frac{\lambda_{\ell j}}{\lambda_{\ell j}+\rho_\ell}
  Z_{\ell,ij}^2}
  {\lVert a_{\ell i}^c\rVert_2^2}.
\]
As implemented, each denominator is floored at \(10^{-12}\); a zero
centered-answer row therefore contributes zero rather than an undefined ratio.
Here \(\eta=0.9\),
\(\rho_\ell=0.1p_\ell^{-1}\sum_{j=1}^{p_\ell}\lambda_{\ell j}\), and
\(\Agg\) denotes the mean of the largest
\(\max\{1,\lceil m_\ell/4\rceil\}\) token-level values.  The layer scores are
\[
  e_\ell
  =
  \clip_{[0,1]}
  \left(
  \frac{\Agg_i(h_{\ell i})-r_{\ell,\eta}/d}
       {1-r_{\ell,\eta}/d}
  \right),
  \qquad
  g_\ell^\star
  =
  \clip_{[0,1]}\!\left(\Agg_i(g_{\ell i})\right).
\]
Thus \(h_{\ell i}\) uses a hard principal-subspace cutoff, whereas
\(g_{\ell i}\) applies a smooth ridge shrinkage across all context directions.
The quantity \(r_{\ell,\eta}/d\) is the expected projected fraction for a
Haar-uniform random \(r_{\ell,\eta}\)-dimensional subspace of \(\R^d\); it is not estimated
from permuted data. The analysis assumes \(r_{\ell,\eta}<d\); if the selected subspace
spans all \(d\) coordinates, define the cumulative-variance principal-subspace
score as one and omit the baseline correction.
For a fixed set of probed layers \(\mathcal L\), the reported scores are
\[
  s_{\mathrm{variance}}
  =
  \exp\left\{
  \frac{1}{|\mathcal L|}
  \sum_{\ell\in\mathcal L}\log\!\bigl(\clip_{[\epsilon,1]}(e_\ell)\bigr)
  \right\},
\]
\[
  s_{\mathrm{ridge}}
  =
  \exp\left\{
  \frac{1}{|\mathcal L|}
  \sum_{\ell\in\mathcal L}
  \log\!\bigl(\clip_{[\epsilon,1]}(g_\ell^\star)\bigr)
  \right\},
  \qquad \epsilon=10^{-6}.
\]

\section{Proofs of the Main Technical Results}
\label{app:proofs}

This appendix collects the longer proofs omitted from the main presentation.

\subsection{Proof of Lemma~\ref{lem:null-concentration}}
\label{app:null-concentration}

\begin{proof}
For \(j\neq k\), the \((j,k)\) entry of \(S^\pi\) has the same distribution as
\(x_j^\top P x_k/n\), where \(P\) is uniform.  Sampling without replacement
and column centering give
\[
  \E_\pi\!\left[
    \left(\frac{x_j^\top P x_k}{n}\right)^2
  \right]
  =
  \frac{S_{jj}S_{kk}}{n-1}.
\]
The diagonal of \(S^\pi\) equals that of \(S\) deterministically. Summing the
off-diagonal second moments proves the identity. Since
\(\lVert M\rVert_\op\leq\lVert M\rVert_\F\), Markov's inequality applied to
the squared Frobenius norm proves the probability bound. Weyl's inequality
gives the simultaneous eigenvalue statement. Taking \(\delta=\alpha\) under
the quantile convention in Lemma~\ref{lem:null-concentration} gives the
threshold bound.
\end{proof}

\subsection{Proof of Theorem~\ref{thm:spectral-range}}
\label{app:spectral-range}

\begin{proof}
Choose an orthogonal matrix \(R_{\mathrm{loc}}\) whose first coordinate is a
leading eigenvector of \(S\). The first diagonal entry of
\(R_{\mathrm{loc}}^\top S R_{\mathrm{loc}}\) is then \(\lambda_1(S)\).
Column permutation preserves every covariance diagonal entry exactly, so every
reference realization has largest eigenvalue at least \(\lambda_1(S)\).
Consequently \(t_1(XR_{\mathrm{loc}})\geq\lambda_1(S)\), the strict first
comparison fails, and
\(\widehat r_{\mathrm{PA}}(XR_{\mathrm{loc}})=0\).

By the Schur--Horn theorem, there is also an orthogonal matrix
\(R_{\mathrm{flat}}\) for which every diagonal entry of
\(R_{\mathrm{flat}}^\top S R_{\mathrm{flat}}\) equals
\(\bar\lambda=\tr(S)/d\). At this orientation,
\[
  V(R_{\mathrm{flat}}^\top S R_{\mathrm{flat}})
  =
  (\tr S)^2-\frac{(\tr S)^2}{d}.
\]
Lemma~\ref{lem:null-concentration} therefore gives
\(t_k(XR_{\mathrm{flat}})\leq c_\alpha(S,n)\) simultaneously for every
componentwise threshold. The observed eigenvalues remain
\(\lambda_k(S)\). By the definition of \(m_\alpha(S,n)\), the first
\(m_\alpha(S,n)\) observed eigenvalues strictly exceed their thresholds, so
\(\widehat r_{\mathrm{PA}}(XR_{\mathrm{flat}})\geq m_\alpha(S,n)\).
Combining the two orientations proves the result.
\end{proof}

\subsection{Proof of Theorem~\ref{thm:same-spectrum}}
\label{app:same-spectrum}

\begin{proof}
Both observed covariance matrices have spectrum
\((\lambda,0,\ldots,0)\). Moreover,
\(\lVert e_1\rVert=\lVert w\rVert=1\) and
\(\1^\top e_1=\1^\top w=1\). The subgroup of orthogonal transformations that
fixes \(\1\) acts transitively on unit vectors having the same inner product
with \(\1\), so there is an \(R\in\mathcal G_{\mathrm{LN}}\) with
\(R^\top e_1=w\).

For \(X_{\mathrm{loc}}\), every reference realization
has the same spectrum because its only nonzero column retains squared norm
\(n\lambda\). Thus \(t_1=\lambda\), and the strict comparison fails.

For \(X_{\mathrm{mix}}\), the largest covariance diagonal entry is
\(4\lambda/9\), and
\[
  V(S)
  =
  \lambda^2\left(1-\sum_{j=1}^d w_j^4\right)
  =
  \frac{16}{27}\lambda^2.
\]
Lemma~\ref{lem:null-concentration} gives
\[
  t_1
  \leq
  \frac{4\lambda}{9}
  +
  \lambda
  \sqrt{\frac{16}{27(n-1)\alpha}}.
\]
The sample-size condition makes this bound strictly smaller than \(\lambda\),
so the first component is retained. The second observed eigenvalue is zero
and its reference threshold is nonnegative, so the procedure stops after one
component.
\end{proof}

\subsection{Proof of Theorem~\ref{thm:single-view-limit}}
\label{app:single-view-limit}

\begin{proof}
Fix a unit vector \(v\) and choose \(R\in O(d)\) such that \(Re_1=v\).
Exact marginal preservation applied to \(XR\) implies that the first projected empirical
measure of a draw from \(\mathcal K(XR,\cdot)\) equals
\(\widehat\mu_{Xv}\) almost surely. Right-orthogonal equivariance
transfers this
degenerate distributional statement to
\[
  \widehat\mu_{Zv}=\widehat\mu_{Xv}
\]
for \(Z\sim\mathcal K(X,\cdot)\). Apply the argument on a countable dense
subset of the unit sphere and intersect the corresponding probability-one
events. Continuity in \(v\) extends the identity to every direction.
The Cram\'er--Wold theorem then yields
\(\widehat\mu_Z=\widehat\mu_X\). Equality of finite empirical measures with
equal atom weights implies equality of their row multisets, hence \(Z=PX\).
\end{proof}

\bibliography{paper_references}
\end{document}